\documentclass[twoside]{article}

\usepackage[a4paper]{geometry}

\usepackage[utf8]{inputenc} 
\usepackage[T1]{fontenc} 

\usepackage{times}
\usepackage{alltt}
\usepackage{rotating}
\usepackage[leqno]{amsmath}
\usepackage{amsfonts}
\usepackage{amssymb}
\usepackage{amsthm}
\usepackage{MnSymbol}
\usepackage{paralist}
\usepackage{mdwlist}
\usepackage{todonotes}
\usepackage{url}
\usepackage{enumitem}
\usepackage{comment}

\usepackage{tcolorbox}
\newtcolorbox{ccomment}{colback=yellow!5!white,colframe=red!75!black,arc=0mm,title=comment} 
\newtcolorbox{rcomment}{colback=red!5!white,colframe=blue!75!black,arc=0mm,title=report}
\newtcolorbox{acomment}{colback=blue!5!white,colframe=black!75!black,arc=0mm,title=paper}

\newcommand{\onlyrr}[2]{#1}

\usepackage{pgf}
\usepackage{tikz}
\usetikzlibrary{mindmap}
\usetikzlibrary{arrows}
\usetikzlibrary{shapes}
\usetikzlibrary{trees}
\usetikzlibrary{snakes}
\usetikzlibrary{plotmarks}
\usetikzlibrary{matrix}
\usetikzlibrary{shapes.gates.logic.US}

\newcommand{\ontoset}{
\tikzstyle{class}=[shape=rectangle, rounded corners, minimum size=.5cm,draw,font=\sf] 
\tikzstyle{type}=[shape=rectangle,color=blue, minimum size=.3cm,draw,font=\tt] 
\tikzstyle{attr}=[color=green!50!black,minimum size=.5cm,font=\sf] 
\tikzstyle{inst}=[font=\itshape,minimum size=.5cm,draw=blue,thick,font=\sf] 
\tikzstyle{data}=[minimum size=.5cm,fill=blue!20,font=\tt] 
\tikzstyle{restr}=[shape=rectangle,minimum size=.5cm,fill=purple!20,font=\sf]
\tikzstyle{axiom}=[shape=rectangle,minimum size=.5cm,fill=orange!20,font=\it]

\tikzstyle{super}=[>=latex,->,thick]
\tikzstyle{parent}=[>=latex,-,thick]
\tikzstyle{const}=[>=latex,->,thin, densely dotted]
\tikzstyle{excl}=[decorate,decoration=zigzag,thick]
\tikzstyle{isa}=[>=latex,->,thin]
\tikzstyle{val}=[>=latex,->,thin]

\tikzstyle{corresp}=[color=blue,<->,very thick]
\tikzstyle{corrbox}=[rectangle,color=blue,very thick]
}

\newtheorem{definition}{Definition}
\newtheorem{example}{Example}
\newtheorem{theorem}{Theorem}

\newtheorem{property}[theorem]{Property}
\newtheorem{lemma}{Lemma}

\makeatletter
\newcommand{\superimpose}[2]{%
  {\ooalign{$#1\@firstoftwo#2$\cr\hfil$#1\@secondoftwo#2$\hfil\cr}}}
\makeatother
\newcommand\sqin{\ensuremath{\mathrel{\mathpalette\superimpose{{\sqsubset}{-}}}}}

\usepackage{RR}
\RRNo{8652}
\RRdate{December 2014}
\RRauthor{
Jérôme Euzenat
}%

\RRtitle{La catégorie des réseaux d'ontologies}
\RRetitle{The category of networks of ontologies}
\titlehead{The category of networks of ontologies}
\RRabstract{
The semantic web has led to the deployment of ontologies on the web connected through various relations and, in particular, alignments of their vocabularies.
There exists several semantics for alignments which make difficult interoperation between different interpretation of networks of ontologies.
Here we present an abstraction of these semantics which allows for defining the notions of closure and consistency for networks of ontologies independently from the precise semantics.
We also show that networks of ontologies with specific notions of morphisms define categories of networks of ontologies.
}
\RRkeyword{Network of ontologies, Ontology alignment, Alignment semantics, Inconsistency, Distributed system semantics, Category, Pullback}
\RRresume{
Le web sémantique a suscité le déploiement d'ontologies sur le web liées par diverses relations, et, en particulier, l'alignement de leur vocabulaire.
Il existe différentes sémantiques pour les alignments qui rendent difficile l'interopération entre différentes interprétations des réseaux d'ontologies.
Ce rapport présente une abstraction de ces sémantiques qui permet de définir les notions de clôture et de consistance de réseaux d'ontologies indépendamment de leur sémantique précise.
On montre aussi que les réseaux d'ontologies dotés d'homomorphismes spécifiques définissent des catégories de réseaux d'ontologies.
}
\RRmotcle{Réseau d'ontologies, Alignement d'ontologies, Sémantique des alignements, Inconsistance, Sémantique des systèmes distribués, Catégorie, Produit fibré}
\RRprojet{Exmo}
\URRhoneAlpes

\begin{document}
\makeRR   


\section{Motivation}

The semantic web relies on knowledge deployed and connected over the web.
This knowledge is based on ontologies expressed in languages such as RDF, RDF Schema and OWL.
Because of the multiplicity of ontologies, they may be connected through alignments expressing correspondences between their concepts.
This allows for translating assertions across ontologies or merging them.
This can be seen as a network of ontologies related by alignments.

The goal of this report is to provide a formal account of such networks of ontologies.
It aims at abstracting properties of networks so that work relying on such properties can apply independently of their interpretation.
In particular, it can contribute to define algebras or revision operators of networks of ontologies.

We first precisely define what alignments and networks of ontologies are through their syntax (\S\ref{sec:asyntax})
before addressing their semantics.
There is no ``standard'' semantics for networks of ontologies, so we provide here an abstract 
view that aims at covering those which have been proposed so far (\S\ref{sec:abssem}).
Based on this framework, we define the notions of closure from this semantics (\S\ref{sec:cons}).
Finally, we show that networks of ontologies for a category and we exhibit some properties of these categories (\S\ref{sec:categ}).
Relation to other work is discussed (\S\ref{sec:rwork}) before concluding.

\bigskip

In this report, ontologies are considered as logical theories and may contain ground assertions.
The languages used in the semantic web such as RDF or OWL are indeed logics \cite{hitzler2009a,antoniou2012a}.
The semantics of ontologies are only considered in this paper through their sets of models ($\mathcal{M}(o)$) and consequence relation ($\models$).
Such a relation satisfies three properties ($o$, $o'$ are ontologies, i.e., sets of assertions, $\delta$ and $\gamma$ are assertions):
\begin{description*}
\item[extensivity] $\{\delta\}\models\delta$
\item[monotony] if $o\models\delta$ then $o\cup o'\models\delta$
\item[idempotency] if $o\models\delta$ and $o\cup\{\delta\}\models\gamma$ then $o\models\gamma$
\end{description*}

We assume a consequence closure function $Cn^{\omega}(o)=\{\delta\mid o\models\delta\}$.

\section{Alignments and networks of ontologies}\label{sec:networks}\label{sec:asyntax}

Alignments express the correspondences between entities of different ontologies \cite{euzenat2013c}. 
Given an ontology $o$ in a language $L$, we use an \emph{entity language} ($Q_L(o)$) for characterising those entities that will be put in correspondence. 
The entity language can be simply made of all the terms or formulas of the ontology language based on the ontology vocabulary.
It can restrict them to the named terms or, on the contrary, extend them to all the queries that may be expressed on this vocabulary.
Alignments express relations between such entities through a finite set $\Theta$ of relations which are independent from ontology relations.

\begin{definition}[Alignment, correspondence] 
Given two ontologies $o$ and $o'$ with associated entity languages $Q_L$ and $Q_{L'}$ and a set of alignment 
relations $\Theta$, a 
\emph{correspondence} is  a triple:
$\langle e, e', r \rangle\in Q_L(o)\times Q'_{L'}(o')\times \Theta$ expressing that the relation $r$ holds between entity $e$ and $e'$.
An alignment is a set of correspondences between two ontologies.
\end{definition}

For the sake of examples, we will consider that ontologies are description logic theories (T-box and A-box) and their entities are classes and individuals identified by URIs.
Classes are denoted by lower case letters, sometimes subscripted by an integer referring to their ontology, and individuals are denoted by $i$.
In the examples, we will only consider $=$, $\leq$, $\geq$ and $\bot$ for relations of $\Theta$.
They will be interpreted as relations expressing equivalence, subsumption and disjointness between classes. 
However, results are not restricted to these relations.

\begin{example}[Alignment]\label{ex:alignment}
The alignment $A_{1,3}$ of Figure~\ref{fig:aligned} (p.\pageref{fig:aligned}), is described by:
$$\{\langle e_1, f_3, \geq\rangle, \langle b_1, e_3,  \geq\rangle\}$$
also described as:
\begin{align*}
A_{1,3} = & \left\lbrace \begin{array}{r@{~}l@{~~~}r@{~}l}
e_1 &\geq f_3, \\
b_1 &\geq e_3
\end{array} \right\rbrace\\
\end{align*}
\end{example}

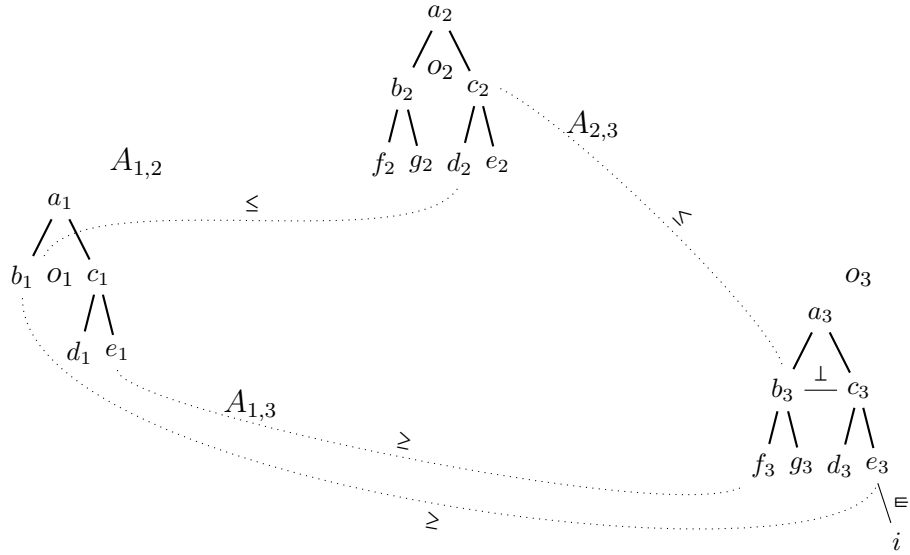
\begin{figure}[h!]
\setlength{\unitlength}{.5cm}%
\begin{center}
\begin{tikzpicture}[level distance=1cm,level/.style={sibling distance=1cm/#1}]
\ontoset;

\draw (0,0.5) node {{\large $o_1$}};
\draw[thick] (0,1.5) node (a1) {$a_1$} 
child {node (b1) {$b_1$}} 
child {node (c1) {$c_1$} 
child {node (d1) {$d_1$}} 
child {node (e1) {$e_1$}}
};

\draw (5,3.25) node {{\large $o_2$}};
\draw[thick] (5,4) node (a2) {$a_2$} 
child {node (b2) {$b_2$} 
child {node (f2) {$f_2$}} 
child {node (g2) {$g_2$}}}
child {node (c2) {$c_2$} 
child {node (d2) {$d_2$}} 
child {node (e2) {$e_2$}}
};

\draw (10.5,.5) node {{\large $o_3$}};
\draw[thick] (10,0) node (a3) {$a_3$} 
child {node (b3) {$b_3$} 
child {node (f3) {$f_3$}} 
child {node (g3) {$g_3$}}} 
child {node (c3) {$c_3$} 
child {node (d3) {$d_3$}} 
child {node (e3) {$e_3$}}
};
\draw (11,-3) node (i) {$i$};
\draw (i) -- node[right] {$\sqin$} (e3.south);
\draw (b3.east) -- node[above] {$\bot$} (c3.west);

\draw (2.5,-1.2) node {{\large $A_{1,3}$}};
\draw[dotted] (b1.south) .. controls +(0,-2.5) and +(-.5,-1.5) .. node[sloped,below] {$\geq$} (e3.south);
\draw[dotted] (e1.south) .. controls +(0,-.5) and +(-1,-.5) .. node[sloped,above] {$\geq$} (f3.south west);
\draw (1,2) node {{\large $A_{1,2}$}};
\draw[dotted] (b1.north east) .. controls +(.5,1) and +(-.5,-1) .. node[sloped,above] {$\leq$} (d2.south);
\draw (7,2.5) node {{\large $A_{2,3}$}};
\draw[dotted] (c2.east) .. controls +(1,-.5) and +(0,.5) .. node[above,sloped] {$\leq$} (b3.north);
%
\end{tikzpicture} 
\end{center}
\caption{A network of ontologies made of three ontologies ($o_1$, $o_2$, and $o_3$) and three alignments ($A_{1,2}$, $A_{1,3}$, and $A_{2,3}$).}\label{fig:aligned}
\end{figure}

The above definition can be generalised to an arbitrary number of alignments and ontologies captured 
in the concept of a network of ontologies (or distributed system \cite{ghidini1998a,franconi2003a}), i.e., sets of ontologies
and alignments.

\begin{definition}[Network of ontologies]\label{def:noo}
A network of ontologies $\langle\Omega, \Lambda\rangle$ is made of a finite set 
$\Omega$\index{$\Omega$ (set of ontologies)|emph} of ontologies and a set 
$\Lambda$\index{$\Lambda$ (set of alignments)|emph} of alignments between these ontologies. We denote by 
$\Lambda(o,o')$ the set of alignments in $\Lambda$ between $o$ and $o'$.
\end{definition}

\begin{example}[Network of ontologies]\label{ex:noo}
Figure~\ref{fig:aligned} presents three ontologies (in all examples, $c \sqsubseteq c'$ denotes subsumption between concepts $c$ and $c'$, $c\bot c'$ denotes disjointness between concepts $c$ and $c'$, and $i\sqin c$ denotes membership of individual $i$ to concept $c$):
\begin{align*}
o_{1} = & \left\lbrace \begin{array}{r@{~}l@{~~~}r@{~}l}
b_1 \sqsubseteq a_1, & c_1 \sqsubseteq a_1\\ 
d_1 \sqsubseteq c_1, & e_1 \sqsubseteq c_1
\end{array} \right\rbrace\\
o_{2} = & \left\lbrace \begin{array}{r@{~}l@{~~~}r@{~}l}
b_2 \sqsubseteq a_2, & c_2 \sqsubseteq a_2, & g_2 \sqsubseteq b_2\\ 
d_2 \sqsubseteq c_2, & e_2 \sqsubseteq c_2, & f_2 \sqsubseteq b_2
\end{array} \right\rbrace\\
o_{3} = & \left\lbrace \begin{array}{r@{~}l@{~~~}r@{~}l}
b_3 \sqsubseteq a_3, & c_3 \sqsubseteq a_3, & g_3 \sqsubseteq b_3\\ 
d_3 \sqsubseteq c_3, & e_3 \sqsubseteq c_3, & f_3 \sqsubseteq b_3\\
i \sqin e_3, & b_3 \bot c_3
\end{array} \right\rbrace\\
\end{align*}
\noindent together with three
alignments $A_{1,2}$, $A_{2,3}$, and $A_{3,1}$. 
These alignments can be described as follows:
\begin{align*}
A_{1,2} = & \left\lbrace \begin{array}{r@{~}l@{~~~}r@{~}l}
b_1 \leq d_2
\end{array} \right\rbrace\\
A_{2,3} = & \left\lbrace \begin{array}{r@{~}l@{~~~}r@{~}l}
c_2 \leq b_3
\end{array} \right\rbrace\\
A_{1,3} = & \left\lbrace \begin{array}{r@{~}l@{~~~}r@{~}l}
e_1 \geq f_3, & b_1 \geq e_3
\end{array} \right\rbrace\\
\end{align*}
\end{example}

Hereafter, we consider normalised networks of ontologies, i.e., networks with exactly one alignment between each pair of ontologies.

\begin{definition}[Normalised network of ontologies]
A network of ontologies $\langle\Omega, \Lambda\rangle$ is said normalised if and only if for any two ontologies $o$ and $o'$, $|\Lambda(o,o')|=1$.
\end{definition}
In a normalised network of ontologies, we denote by $\lambda(o,o')$ the unique alignment between $o$ and $o'$.

\onlyrr{
Any network of ontologies may easily be normalised by:
\begin{itemize*}
\item if $|\Lambda(o,o')|=0$, adding an empty alignment between $o$ and $o'$,
\item if $|\Lambda(o,o')|>1$, replacing $\Lambda(o,o')$ by a unique alignment containing all the correspondences of the alignments of $\Lambda(o,o')$,
\end{itemize*}
We call this standard normalisation:
\begin{definition}[Standard normalisation]\label{def:normcorr}
Given a network ontology $\langle\Omega,\Lambda\rangle$, its standard normalisation $\langle\Omega,\overline{\Lambda}\rangle$ is
defined by $\forall o, o'\in\Omega$, 
$$\overline{\Lambda}(o,o')=
\begin{cases}
\{\{\}\}, & \text{if } \Lambda(o,o')=\emptyset,\\
\{\bigcup_{A\in\Lambda(o,o')}A\}, & \text{otherwise}
\end{cases}$$
The unique element of $\overline{\Lambda}(o,o')$ is denoted by $\overline{\lambda}(o,o')$.
\end{definition}

There are other ways to normalise such networks, but this simple one is sufficient for obtaining equivalent normalised networks (see Property~\ref{prop:normcorr}).
}{}

Comparing networks of ontologies
is not in general simple.
For that purpose, we introduce the notion of morphism between two networks of ontologies.

\begin{definition}[Syntactic morphism between networks of ontologies]\label{def:synmorph}
Given two networks of ontologies, $\langle\Omega, \Lambda\rangle$ and $\langle \Omega',\Lambda'\rangle$,
a \emph{syntactic morphism} between $\langle\Omega, \Lambda\rangle$ and $\langle \Omega',\Lambda'\rangle$, is
$\langle h, k\rangle$, a pair of morphisms: $h:\Omega\longrightarrow\Omega'$ and $k:\Lambda\longrightarrow\Lambda'$ such that
$\forall o\in\Omega$, $\exists h(o)\in\Omega'$ and $o\subseteq h(o)$ and 
$\forall A\in\Lambda(o,o')$, $\exists k(A)\in \Lambda'(h(o),h(o'))$ and $A\subseteq k(A)$.
\end{definition}

Such a morphism exists when one network can be projected into another one an keep its structure and syntactic content.
This means that any ontology (respectively any alignment) of the former has a counterpart in the latter one which contains at least all of its axioms (respectively correspondences) and that the graph structure of the former network is preserved in the latter.
It is possible that several ontologies or alignments have the same counterpart as long as these conditions are met.

Morphisms can be used for defining syntactic subsumption between networks of ontologies.

\begin{definition}[Syntactic subsumption between networks of ontologies]\label{def:netsubs}
Given two networks of ontologies, $\langle\Omega, \Lambda\rangle$ and $\langle \Omega',\Lambda'\rangle$,
$\langle\Omega, \Lambda\rangle$ is \emph{syntactically subsumed} by $\langle \Omega',\Lambda'\rangle$, denoted by $\langle\Omega, \Lambda\rangle\sqsubseteq\langle \Omega',\Lambda'\rangle$, 
iff
there exist a syntactic morphism between $\langle\Omega, \Lambda\rangle$ and $\langle \Omega',\Lambda'\rangle$.
\end{definition}

\onlyrr{

We note: 
\begin{align*}
\langle\Omega, \Lambda\rangle\equiv\langle \Omega',\Lambda'\rangle &\text{ iff } \langle\Omega, \Lambda\rangle\sqsubseteq\langle \Omega',\Lambda'\rangle \text{ and }\langle\Omega', \Lambda'\rangle\sqsubseteq\langle \Omega,\Lambda\rangle\\
\langle\Omega, \Lambda\rangle\sqsubset\langle \Omega',\Lambda'\rangle&\text{ iff } \langle\Omega, \Lambda\rangle\sqsubseteq\langle \Omega',\Lambda'\rangle \text{ and }\langle\Omega', \Lambda'\rangle\not\sqsubseteq\langle \Omega,\Lambda\rangle
\end{align*}
}{
We note: 
$\langle\Omega, \Lambda\rangle\equiv\langle \Omega',\Lambda'\rangle$ iff $\langle\Omega, \Lambda\rangle\sqsubseteq\langle \Omega',\Lambda'\rangle$ and $\langle\Omega', \Lambda'\rangle\sqsubseteq\langle \Omega,\Lambda\rangle$;
$\langle\Omega, \Lambda\rangle\sqsubset\langle \Omega',\Lambda'\rangle$ iff $\langle\Omega, \Lambda\rangle\sqsubseteq\langle \Omega',\Lambda'\rangle$ and $\langle\Omega', \Lambda'\rangle\not\sqsubseteq\langle \Omega,\Lambda\rangle$.
}

This definition is purely syntactic because semantically equivalent networks may not be syntactically subsumed (it suffices to use one ontology whose axioms are equivalent but different).

The empty network of ontologies $\langle \emptyset, \emptyset\rangle$ (containing no ontology and no alignment) is subsumed by any other network of ontologies.

\onlyrr{

It is possible to simplify the above definition in case of normalised networks of ontologies.
\begin{property}
Given two normalised networks of ontologies $\langle\Omega, \Lambda\rangle$ and $\langle\Omega', \Lambda'\rangle$,
$\langle\Omega, \Lambda\rangle \sqsubseteq\langle\Omega', \Lambda'\rangle$ iff $\exists h:\Omega\rightarrow\Omega'$
a morphism such that 
$\forall o\in\Omega$, $\exists h(o)\in\Omega'$ and $o\subseteq h(o)$ and 
$\forall o, o'\in\Omega$, $\lambda(o,o')\subseteq\lambda'(h(o),h(o'))$.
\end{property}
\begin{proof}
$\Rightarrow$) 
In normalised networks, there always exists a single alignment between each pair of ontologies.
Hence, if $k$ is such that $\forall A\in\Lambda(o,o')$, $k(A)\in \Lambda'(h(o),h(o'))$, this means that $k(A)=k(\lambda(o,o'))=\lambda'(h(o),h(o'))$
and since $A\subseteq k(A)$, then $\lambda(o,o')\subseteq\lambda'(h(o),h(o'))$.

$\Leftarrow$)
In $\langle\Omega, \Lambda\rangle$, $\Lambda(o,o')=\{\lambda(o,o') \}$ and the same holds for $\langle\Omega', \Lambda'\rangle$, 
if $k(\lambda(o,o'))=\lambda'(h(o),h(o'))$ it satisfies the constraint that $k(A)\in \Lambda'(h(o),h(o'))$ and $A\subseteq k(A)$.
\end{proof}

Moreover, networks are subsumed by their standard normalisation.

\begin{property} Let $\langle\Omega,\Lambda\rangle$ a network of ontologies and $\langle\Omega,\overline{\Lambda}\rangle$ its standard normalisation,
$$\langle\Omega,\Lambda\rangle \sqsubseteq \langle\Omega,\overline{\Lambda}\rangle$$
\end{property}
\begin{proof}
Consider, the pair of morphisms $\langle h, k\rangle$ such that $h(o)=o$ and $\forall o,o'\in\Omega, \forall A\in\Lambda(o,o')$, $k(A)=\overline{\lambda}(o,o')$,
then $A\subseteq k(A)$ because $A\subseteq \bigcup_{A'\in\Lambda(o,o')}A'$.
\end{proof}


From subsumption, conjunction (meet) can be introduced in a standard way:
\begin{definition}[Syntactic conjunction of networks of ontologies]
Given a finite family of networks of ontologies, $\{\langle\Omega_i, \Lambda_i\rangle\}_{i\in I}$,
$\bigsqcap_{i\in I}\langle \Omega_i,\Lambda_i\rangle=\langle \Omega',\Lambda'\rangle$
such that $\forall i\in I$, $\langle \Omega',\Lambda'\rangle\sqsubseteq\langle\Omega_i, \Lambda_i\rangle$ and
$\forall \langle \Omega'',\Lambda''\rangle$; $\langle \Omega'',\Lambda''\rangle\sqsubseteq\langle\Omega_i, \Lambda_i\rangle$, $\langle \Omega'',\Lambda''\rangle\sqsubseteq\langle\Omega', \Lambda'\rangle$.
\end{definition}
Such a conjunction always exists because the empty network of ontologies is subsumed by all network of ontologies.
However, it is likely not unique because of the choice of homomorphisms.

We can also define simple operations on networks of ontologies.

\begin{definition}[Substitution in networks of ontologies]
Given a network of ontologies $\langle\Omega, \Lambda\rangle$, given $o\in\Omega$ and $o'$ another ontology,
$\langle\Omega, \Lambda\rangle[o/o']=\langle\Omega\setminus\{o\}\cup\{o'\}, \Lambda\setminus\bigcup_{o''\in\Omega} (\Lambda(o,o'')\cup\Lambda(o'',o))\rangle$.

Given a network of ontologies $\langle\Omega, \Lambda\rangle$, given $A\in\Lambda(o,o')$ and $A'$ another alignment between $o$ and $o'$,
$\langle\Omega, \Lambda\rangle[A/A']=\langle\Omega, \Lambda\setminus\{A\}\cup\{A'\}\rangle$.
\end{definition}

\begin{property}
Given a network of ontologies $\langle\Omega, \Lambda\rangle$ with $o\in\Omega$ and $A\in\Lambda$, 
if $o'\subseteq o$, then $\langle\Omega, \Lambda\rangle[o/o']\sqsubseteq\langle\Omega, \Lambda\rangle$ and if $A'\subseteq A$, then $\langle\Omega, \Lambda\rangle[A/A']\sqsubseteq\langle\Omega, \Lambda\rangle$
\end{property}
\begin{proof}
Simply, there exists a pair of morphisms $\langle h, k\rangle$ which identifies each ontology to itself but $o'$ which is identified to $o$, and each alignment to itself, but $A'$ which is identified to $A$ in the second case.
In the case of $\langle\Omega, \Lambda\rangle[o/o']$, the set of alignments is strictly included in $\Lambda$.
It is thus clear that $h(o'')\subseteq o''$ and $k(A'')\subseteq A''$, including for $o'$ and $A'$ (by hypothesis).
The structure is preserved because ontologies and alignments are the same, except in the second case in which the substituted alignment preserves the structure.
\end{proof}

}{}


\section{Semantics of networks of ontologies}\label{sec:abssem}

The semantics of aligned ontologies, or networks of ontologies, must remain compatible with the classical semantics of ontologies: connecting ontologies to other ontologies should not radically change the manner to interpret them.

When ontologies are independent, i.e., not related with alignments, it is natural that their semantics is the classical semantics for these ontologies, i.e., a set of models $\mathcal{M}(o)$.
A model is a map $m$ from the entities of the ontologies to a particular domain $D$. 
Such models have to apply to all the elements of the entity language $Q_L(o)$ (when it is larger than the ontology language, this is usually defined inductively on the structure of its elements).

Different semantics provide alternative ways to record the constraints imposed by alignments: through relations between domains of interpretation \cite{ghidini1998a,borgida2003a}, through equalising functions \cite{zimmermann2006b,zimmermann2008c}, by imposing equal \cite{lenzerini2002a} or disjoint \cite{cuencagrau2006a} domains.
These models have been compared elsewhere \cite{zimmermann2006b}; we provide an informal unified view of these semantics.

For that purpose, each correspondence is interpreted with respect to three features: 
a model for each ontology and a semantic structure, denoted by $\Delta$ \cite{zimmermann2013a}.
This loosely defined semantic structure has two purposes:
\begin{itemize*}
\item providing an interpretation to the correspondence relations in $\Theta$ (which are independent from the ontology semantics);
\item memorising the constraints imposed on models by the alignments.
\end{itemize*}
In this work, it can simply be considered that $\Delta$ is used, in each semantics, to define the satisfaction of a correspondence $\mu$ by two ontology models $o$ and $o'$ (which is denoted by $m_o, m_{o'}\models_{\Delta} \mu$).


%

Such a simple notion of satisfaction, imposing no constraints on models, is provided by Example~\ref{ex:relations}.

\begin{example}[Interpretation of correspondences]\label{ex:relations}
In the language used as example, $c$ and $c'$ stand for classes and $i$ and $i'$ for individuals.
If $m_o$ and $m_{o'}$ are respective models of $o$ and $o'$:
\begin{align*}
m_o, m_{o'}\models_{\Delta} \langle c, c', =\rangle &\text{ iff } m_o(c)= m_{o'}(c')\\
m_o, m_{o'}\models_{\Delta} \langle c, c', \leq\rangle &\text{ iff } m_o(c)\subseteq m_{o'}(c')\\
m_o, m_{o'}\models_{\Delta} \langle c, c', \geq\rangle &\text{ iff } m_o(c)\supseteq m_{o'}(c')\\
m_o, m_{o'}\models_{\Delta} \langle i, c', \in\rangle &\text{ iff } m_o(i)\in m_{o'}(c')\\
m_o, m_{o'}\models_{\Delta} \langle c, i', \ni\rangle &\text{ iff } m_{o'}(i')\in m_o(c)\\
m_o, m_{o'}\models_{\Delta} \langle c, c', \bot\rangle &\text{ iff } m_o(c)\cap m_{o'}(c')=\emptyset 
\end{align*}
\end{example}

\onlyrr{
\begin{example}[Interpretation of correspondences in first-order logic]\label{ex:folint}
The semantics can be given with respect to first-order logic theories.
In such a case, correspondences relate predicates $p$ and $p'$ of the same arity between ontologies $o$ and $o'$.
If $m_o$ and $m_{o'}$ are their respective first-order models:
\begin{align*}
m_o, m_{o'}\models_{\Delta} \langle p, p', =\rangle &\text{ iff } m_o(p)=m_{o'}(p')\\
m_o, m_{o'}\models_{\Delta} \langle p, p', \leq\rangle &\text{ iff } m_o(p)\subseteq m_{o'}(p')\\
m_o, m_{o'}\models_{\Delta} \langle p, p', \geq\rangle &\text{ iff } m_o(p)\supseteq m_{o'}(p')\\
m_o, m_{o'}\models_{\Delta} \langle p, p', \bot\rangle &\text{ iff } m_o(p)\cap m_{o'}(p')=\emptyset 
\end{align*}
This semantics selects first-order theory interpretations by setting constraints on predicate interpretations.
\end{example}

\begin{example}[Interpretation of correspondences in the equalising semantics]\label{ex:equalisingint}
An alternative interpretation in the equalising semantics \cite{zimmermann2008c} relies on a family of functions 
indexed by each ontology  $\gamma$ from the domains of interpretations of each ontologies to a universal domain $U$.
Hence, $\Delta=\langle \gamma, U\rangle$.
Then, in the case of Example~\ref{ex:relations}, if $m_o$ and $m_{o'}$ are respective models of $o$ and $o'$:
\begin{align*}
m_o, m_{o'}\models_{\Delta} \langle c, c', =\rangle &\text{ iff } \gamma_o\circ m_o(c)= \gamma_{o'}\circ m_{o'}(c')\\
m_o, m_{o'}\models_{\Delta} \langle c, c', \leq\rangle &\text{ iff } \gamma_o\circ m_o(c)\subseteq \gamma_{o'}\circ m_{o'}(c') \text{ or } \gamma_o\circ m_o(c)\in \gamma_{o'}\circ m_{o'}(c')\\
m_o, m_{o'}\models_{\Delta} \langle c, c', \geq\rangle &\text{ iff } \gamma_o\circ m_o(c)\supseteq \gamma_{o'}\circ m_{o'}(c') \text{ or }\gamma_{o'}\circ m_{o'}(c')\in \gamma_{o}\circ m_{o}(c)\\
m_o, m_{o'}\models_{\Delta} \langle i, c', \in\rangle &\text{ iff } \gamma_o\circ m_o(i)\in \gamma_{o'}\circ m_{o'}(c') \text{ or } \gamma_o\circ m_o(i)\subseteq \gamma_{o'}\circ m_{o'}(c')
\end{align*}
This semantics allows for changing the interpretation of an individual as a set and vice-versa.
\end{example}
}{}


Hence, the semantics of two aligned ontologies may be given as a set of models which are pairs of compatible models. 

\begin{definition}[Models of alignments]\index{model!alignment -|emph}\index{alignment!model|emph}
Given two ontologies $o$ and $o'$ and an alignment $A$ between these ontologies, a model
of this alignment is a triple
$\langle m_o, m_{o'}, \Delta\rangle$ with $m_o\in\mathcal{M}(o)$, $m_{o'}\in\mathcal{M}(o')$, and $\Delta$ a semantic structure,
such that $\forall \mu\in A$, $m_o, m_{o'}\models_{\Delta} \mu$ (denoted by $m_o, m_{o'}\models_{\Delta} A$).
\end{definition}

We note $A\models\mu$ iff $\forall \langle m_o, m_{o'}, \Delta\rangle$ such that $m_o, m_{o'}\models_{\Delta} A$, $m_o, m_{o'}\models_{\Delta}\mu$.
Similarly as for ontologies, the semantics of alignments can be given by the relation $\models$ such that ($A$ and $A'$ are alignments and $\mu$ and $\nu$ are correspondences all between the same pair of ontologies):
\begin{description*}
\item[extensivity] $\{\mu\}\models\mu$
\item[monotony] if $A\models\mu$ then $A\cup A'\models\mu$
\item[idempotency] if $A\models\mu$ and $A\cup\{\mu\}\models\nu$ then $A\models\nu$
\end{description*}
Similarly, we assume a consequence closure function $Cn^{\alpha}(A)=\{\mu\mid A\models\mu\}$.

Models of networks of ontologies extend models of alignments.
They select compatible models for each ontology in the network \cite{ghidini2001a}.
Compatibility consists of satisfying all the alignments of the network.

\begin{definition}[Models of networks of ontologies]\label{def:modelnoe}
Given a network of ontologies
$\langle\Omega,\Lambda\rangle$, a model of $\langle\Omega,\Lambda\rangle$ is a pair
$\langle m, \Delta\rangle$ with $m$ a family of models indexed by $\Omega$ with $\forall o\in\Omega$, $m_o\in\mathcal{M}(o)$
such that for each
alignment $A\in\Lambda(o,o')$, $m_o, m_{o'}\models_{\Delta} A$.
The set of models of $\langle\Omega,\Lambda\rangle$ is denoted by $\mathcal{M}(\langle\Omega,\Lambda\rangle)$.
\end{definition}

In that respect, alignments act as model filters for the ontologies.
They select the ontology interpretations which are coherent with the alignments.
This allows for transferring information from one ontology to another since reducing the set of models entails more consequences in each aligned ontology.

\begin{example}[Model of a network of ontologies]\label{ex:almodel}
Hence, a model for the network of ontologies of Figure~\ref{fig:aligned} with $\Delta$ as defined in Example~\ref{ex:relations}, is $\langle \{m_1, m_2, m_3\}, \Delta\rangle$ built on any models $m_1$, $m_2$ and $m_3$ of ontology $o_1$, $o_2$ and $o_3$ such that
$m_3(e_3)\subseteq m_1(b_1)\subseteq m_2(d_2)$, $m_2(c_2)\subseteq m_3(b_3)$ and $m_3(f_3)\subseteq m_1(e_1)$.
\end{example}

\onlyrr{

Standard normalisation provides a semantically equivalent network according to this semantics.
\begin{property}[Soundness of standard normalisation]\label{prop:normcorr}
Let $\langle\Omega, \Lambda\rangle$ be a network of ontologies and $\langle\Omega, \overline{\Lambda}\rangle$ its standard normalisation,
$$\mathcal{M}(\langle\Omega,\Lambda\rangle) = \mathcal{M}(\langle\Omega,\overline{\Lambda}\rangle)$$
\end{property}
\begin{proof}
$\mathcal{M}(\langle\Omega,\Lambda\rangle) = \mathcal{M}(\langle\Omega,\overline{\Lambda}\rangle)$
because each model is made of a set of models of the same ontologies $\Omega$ and a semantic structure $\Delta$ which has to satisfy the same set of correspondences.

More precisely, if $\langle m, \Delta\rangle\in\mathcal{M}(\langle\Omega,\Lambda\rangle)$,
then $\forall o, o'\in \Omega, \forall A\in\Lambda(o,o'), \forall \mu\in A, m_o, m_{o'}\models_{\Delta}\mu$ 
which means that $\forall o, o'\in \Omega, \forall \mu\in\overline{\lambda}(o,o'), m_o, m_{o'}\models_{\Delta}\mu$
because either $\Lambda(o,o')=\varnothing$ and then $\overline{\lambda}(o,o')$ does not contain any $\mu$
or $\forall \mu\in\overline{\lambda}(o,o'), \exists A\in\Lambda(o,o')$ such that $\mu\in A$ and thus $m_o, m_{o'}\models_{\Delta}\mu$.
Hence, $\langle m, \Delta\rangle\in\mathcal{M}(\langle\Omega,\overline{\Lambda}\rangle)$

In the reverse direction, if $\langle m, \Delta\rangle\in\mathcal{M}(\langle\Omega,\overline{\Lambda}\rangle)$,
then $\forall o, o'\in \Omega$, $\forall \mu\in\overline{\lambda}(o,o')$, $m_o, m_{o'}\models_{\Delta}\mu$,
but $\forall o, o'\in \Omega, \forall A\in\Lambda(o,o'), \forall \mu\in A$, $\mu\in\overline{\lambda}(o,o')$ because $\overline{\lambda}(o,o')=\bigcup_{A\in\Lambda(o,o')}A$,
thus $m_o, m_{o'}\models_{\Delta}\mu$.
Hence, $\langle m, \Delta\rangle\in\mathcal{M}(\langle\Omega,\Lambda\rangle)$.
\end{proof}
This justifies the position to only consider normalised networks of ontologies.

We consider an order relation $\ll$ between semantic structures denoting the reinforcement of constraints (more complete relations between domains of interpretations or more disjunctions between domains).
The stronger the semantic structure, the less models it accepts.
\begin{definition}[Constraint reinforcement]\label{def:deltaorder}
Given two semantic structures $\Delta$ and $\Delta'$,
$$\Delta \ll \Delta'\text{ iff } \forall m, m', \forall\mu, m, m'\models_{\Delta'}\mu \Rightarrow m, m'\models_{\Delta}\mu$$
\end{definition}

The models of a network and that of its standard normalisation can be compared because they are indexed by the same set of ontologies.
However, it is not generally possible to directly compare two sets of models of two networks because the set of ontologies that index them is not the same.
Therefore, this is again done up to homomorphism.

\begin{definition}[Model inclusion]\label{def:modelnoeincl}
Given two networks of ontologies $\langle \Omega,\Lambda\rangle$ and $\langle \Omega',\Lambda'\rangle$,
the set of models of the latter is said included in that of the former, and denoted by $\mathcal{M}(\langle \Omega',\Lambda'\rangle)\trianglelefteq\mathcal{M}(\langle \Omega,\Lambda\rangle)$, if and only if
there exists a map $h: \Omega\longrightarrow\Omega'$ such that $\forall \langle m', \Delta'\rangle\in\mathcal{M}(\langle\Omega',\Lambda'\rangle)$, $\exists \langle m, \Delta\rangle\in\mathcal{M}(\langle\Omega,\Lambda\rangle)$; $\forall o\in\Omega, m_{o}=m'_{h(o)}$ and $\Delta\ll\Delta'$.
\end{definition}

Property~\ref{prop:monomod} shows that syntactically subsumed networks of ontologies have more models.

\begin{property}[Model antitony]\label{prop:monomod}
Let $\langle\Omega, \Lambda\rangle$ and $\langle\Omega', \Lambda'\rangle$ be two networks of ontologies,
$$\langle\Omega,\Lambda\rangle \sqsubseteq \langle\Omega',\Lambda'\rangle \Rightarrow \mathcal{M}(\langle\Omega',\Lambda'\rangle)\trianglelefteq \mathcal{M}(\langle\Omega,\Lambda\rangle)$$
\end{property}
\begin{proof}
$\langle\Omega,\Lambda\rangle \sqsubseteq \langle\Omega',\Lambda'\rangle$ 
means that
$\exists\langle h,k\rangle$; $\forall o\in\Omega, o\subseteq h(o)$ and $\forall o, o'\in \Omega$, $\forall A\in\Lambda(o,o')$, $A\subseteq k(A)\wedge k(A)\in\Lambda'(h(o),h(o'))$.
Hence, $\exists\langle h,k\rangle$; $\forall m\in\mathcal{M}(h(o))$, $m\in\mathcal{M}(o)$ and $\forall \langle m,m'\rangle\in\mathcal{M}(h(o))\times\mathcal{M}(h(o'))$, 
$\forall \Delta'$, $m, m'\models_{\Delta'} k(A)\Rightarrow m,m'\models_{\Delta'} A$ (because $A\subseteq k(A)$).
In addition, because $A$ imposes less constraints on the models than $k(A)$, models may be defined with $\Delta\ll\Delta'$,
but in such a case, $m,m'\models_{\Delta} A$.
Thus, there exists a map $h: \Omega\longrightarrow\Omega'$ such that $\forall \langle m', \Delta'\rangle\in\mathcal{M}(\langle\Omega',\Lambda'\rangle)$, $\exists \langle m, \Delta\rangle\in\mathcal{M}(\langle\Omega,\Lambda\rangle)$; $\forall o\in\Omega, m_{o}=m'_{h(o)}$ and $\Delta\ll\Delta'$.
In consequence, $\mathcal{M}(\langle \Omega',\Lambda'\rangle)\trianglelefteq\mathcal{M}(\langle \Omega,\Lambda\rangle)$.
\end{proof}

\begin{property}[Downward consistency preservation]\label{prop:concur}
Let $\langle\Omega, \Lambda\rangle$ and $\langle\Omega', \Lambda'\rangle$ be two networks of ontologies,
If $\langle\Omega,\Lambda\rangle \sqsubseteq \langle\Omega',\Lambda'\rangle$
and $\langle\Omega',\Lambda'\rangle$ is consistent,
then $\langle\Omega,\Lambda\rangle$ is consistent.
\end{property}
\begin{proof}
Straightforward from Property~\ref{prop:monomod}, since if $\langle\Omega',\Lambda'\rangle$ is consistent, it has a model, and so does $\langle\Omega,\Lambda\rangle$.
\end{proof}
}{
}

It is expected that all constraints applying to the semantics are preserved by the syntactic morphisms.
However, the converse is not guarantee: if a network preserves the constraints of another then there does not necessarily imply that there exist a syntactic morphism.

\section{Consistency, consequence and closure}\label{sec:cons}

A network of ontologies is consistent if it has a model. 
By extension, an ontology or an alignment is consistent within a network of ontologies if the network of ontologies is consistent. 
Hence even if an ontology is consistent when taken in isolation, it may be inconsistent when inserted in a network of ontologies.
Moreover, if one of the ontologies in the network is inconsistent, then the network as a whole is inconsistent.

\begin{example}[Inconsistency]\label{ex:inconsistency}
A model of the network of ontologies presented in Example~\ref{ex:noo} according to the interpretation of correspondence of Example~\ref{ex:relations}, retains families of models $\{m_1, m_2, m_3\}$
satisfying all alignments, i.e., in particular, satisfying:
\begin{align*}
m_3(b_3) &\supseteq m_2(c_2)\\
m_2(c_2) &\supseteq m_1(b_1)\\
m_1(b_1) &\supseteq m_3(e_3)
\end{align*}
But, all models $m_3$ of $o_3$ must satisfy $m_3(b_3)\cap m_3(c_3)=\emptyset$, $m_3(i)\in m_3(e_3)$, and $m_3(i)\in m_3(c_3)$.
Moreover, all models $m_2$ of $o_2$ must satisfy $m_2(d_2)\subseteq m_2(c_2)$.
Hence, $m_3(b_3)\supseteq m_1(b_1)$, $m_3(b_3)\supseteq m_3(e_3)$ and then $m_3(i)\in m_3(b_3)$, which is contradictory with previous assertions.
Thus, there cannot exists such a family of models and there is no model for this network of ontologies.

If the interpretation of Example~\ref{ex:equalisingint} were retained, there may be models in which $\gamma\circ m_3(i)=\emptyset$.

In this example, taking any of the ontologies with only the alignments which involve them, e.g., $\langle\Omega, \{A_{1,3}, A_{2,3}\}\rangle$, is a consistent network of ontologies.
The following examples will also consider the network of ontologies $\langle\Omega',\Lambda\rangle$=$\langle\{o_1, o_2, o'_3\},$ $\{A_{1,2}, A_{1,3}, A_{2,3}\}\rangle$ such that $o'_3$ is $o_3\setminus\{i \sqin e_3\}$.
\end{example}
So far, we have not defined what it means for a formula to be the consequence of a network. 
There are two notions of consequences called $\omega$-consequence and $\alpha$-consequence.

$\alpha$-consequences are correspondences which are consequences of networks of ontologies \cite{euzenat2007a}. 

\begin{definition}[$\alpha$-Consequence of networks of ontologies]\label{def:alphacons}
Given a finite set of ontologies $\Omega$ and a finite set of alignments $\Lambda$
between pairs of ontologies in $\Omega$, a correspondence $\mu$ between two ontologies $o$ and $o'$ in $\Omega$ is an $\alpha$-consequence of $\langle\Omega,\Lambda\rangle$ (denoted by $\models_{\Omega,\Lambda} \mu$ or $\langle\Omega,\Lambda\rangle\models \mu$)
if and only if for all models $\langle m, \Delta\rangle$ of $\langle\Omega,\Lambda\rangle$, $m_o, m_{o'}\models_{\Delta} \mu$.
\end{definition}

The set of $\alpha$-consequences between $o$ and $o'$ is denoted by $Cn^{\alpha}_{\Omega,\Lambda}(o,o')$. 
For homogeneity of notation, we will use $Cn^{\alpha}_{\Omega,\Lambda}(A)$ for denoting $Cn^{\alpha}_{\Omega,\Lambda}(o,o')$ when $A\in\Lambda(o,o')$.
The $\alpha$-closure of a network of ontologies is its set of $\alpha$-consequences: the correspondences which are satisfied in all models of the network of ontologies.

From the alignment semantics, it is possible to decide if an alignment is a consequence of another or if the alignment makes the set of ontologies and alignments inconsistent.

\begin{example}[$\alpha$-consequences]\label{ex:alpha}
The closure of $A_{1,3}$ in the network of ontology $\langle\Omega', \Lambda\rangle$ of Example~\ref{ex:inconsistency} is:
$$Cn^{\alpha}_{\Omega',\Lambda}(o_1,o'_3)= \left\lbrace \begin{array}{r@{~}l@{~~~}r@{~}l}
e_1 \geq f_3, & b_1 \geq e_3\\
c_1 \geq f_3, & a_1 \geq f_3\\
a_1 \geq e_3, & b_1 \leq b_3\\
b_1 \leq a_3, & b_1 \bot c_3\\
b_1 \bot d_3, & b_1 \bot e_3
\end{array} \right\rbrace$$
but if the network is reduced to the two involved ontologies ($o_1$ and $o'_3$) only, the closure would be:
$$Cn^{\alpha}_{\{o_1,o'_3\},\{A_{1,3}\}}(o_1,o'_3)= \left\lbrace \begin{array}{r@{~}l@{~~~}r@{~}l}
e_1 \geq f_3, & b_1 \geq e_3\\
c_1 \geq f_3, & a_1 \geq f_3\\
a_1 \geq e_3
\end{array} \right\rbrace$$
It is thus clear that connecting more ontologies provides more information.
\end{example}

According to these definitions, $Cn^{\alpha}(A)=Cn^{\alpha}_{\langle \{o,o'\}, \{A\}\rangle}(A)$ when $A\in\Lambda(o,o')$.
$\alpha$-consequences of an alignment are defined as the $\alpha$-consequences of the network made of this alignment and the two ontologies it connects.
The $\alpha$-consequences of a particular alignment are usually larger than the alignment ($\forall A\in\Lambda, A\subseteq Cn^{\alpha}(A)\subseteq Cn^{\alpha}_{\Omega,\Lambda}(A)$). 
If the alignment is not satisfiable, then any correspondence is one of its $\alpha$-consequences.


Similarly, the
$\omega$-consequences of an ontology in a network are formulas that are satisfied in all models of the ontology selected by the network.

\begin{definition}[$\omega$-Consequence of an ontology in a network of ontologies]
Given a finite set of ontologies $\Omega$ and a finite set of alignments $\Lambda$ 
between pairs of ontologies in $\Omega$, a formula $\delta$ in the ontology language of $o\in\Omega$ is an $\omega$-consequence of $o$ in $\langle\Omega,\Lambda\rangle$ (denoted by $o\models_{\Omega,\Lambda} \delta$) if 
and only if for all models $\langle m, \Delta\rangle$ of $\langle\Omega,\Lambda\rangle$, $m_o\models \delta$ (the set of $\omega$-consequences of $o$ is denoted by $Cn^{\omega}_{\Omega,\Lambda}(o)$).
\end{definition}

The $\omega$-closure of an ontology is the set of its $\omega$-consequences.
According to these definitions, $Cn^{\omega}(o)=Cn^{\omega}_{\langle \{o\}, \emptyset\rangle}(o)$.
These $\omega$-consequences are
larger than the classical consequences of the ontology ($\forall o\in\Omega, o\subseteq Cn^{\omega}(o)\subseteq Cn^{\omega}_{\Omega,\Lambda}(o)$)
because they rely on a smaller set of models.

\begin{example}[$\omega$-consequences]\label{ex:omega}
The simple consequences of the ontology $o'_3$ are:
$$Cn^{\omega}(o'_{3}) = \left\lbrace \begin{array}{r@{~}l@{~~~}r@{~}l}
b_3 \sqsubseteq a_3, & c_3 \sqsubseteq a_3, & g_3 \sqsubseteq b_3\\ 
d_3 \sqsubseteq c_3, & e_3 \sqsubseteq c_3, & f_3 \sqsubseteq b_3\\
f_3 \bot c_3, & b_3 \bot c_3 & f_3 \sqsubseteq a_3,\\
d_3 \sqsubseteq a_3, & e_3 \sqsubseteq a_3,  & g_3 \bot c_3\\
 d_3 \bot b_3, & e_3 \bot b_3 & g_3 \sqsubseteq a_3
\end{array} \right\rbrace$$
\noindent while within $\langle \Omega', \Lambda\rangle$ of Example~\ref{ex:inconsistency}, there are even more consequences:
$$Cn^{\omega}_{\Omega',\Lambda}(o'_{3}) = \left\lbrace \begin{array}{r@{~}l@{~~~}r@{~}l}
b_3 \sqsubseteq a_3, & c_3 \sqsubseteq a_3, & g_3 \sqsubseteq b_3\\ 
d_3 \sqsubseteq c_3, & e_3 \sqsubseteq c_3, & f_3 \sqsubseteq b_3\\
f_3 \bot c_3, & b_3 \bot c_3 & f_3 \sqsubseteq a_3,\\
d_3 \sqsubseteq a_3, & e_3 \sqsubseteq a_3, & g_3 \bot c_3,  \\
d_3 \bot b_3, & e_3 \bot b_3 & g_3 \sqsubseteq a_3\\
& b_3 \sqsupseteq e_3\\
\end{array} \right\rbrace$$
\end{example}

\onlyrr{
From the notion of consequence, we introduce semantic morphism.

\begin{definition}[Semantic morphism between networks of ontologies]\label{def:semmorph}
Given two networks of ontologies, $\langle\Omega, \Lambda\rangle$ and $\langle \Omega',\Lambda'\rangle$,
a \emph{semantic morphism} between $\langle\Omega, \Lambda\rangle$ and $\langle \Omega',\Lambda'\rangle$, is
$\exists \langle h, k\rangle$, a pair of morphisms: $h:\Omega\longrightarrow\Omega'$ and $k:\Lambda\longrightarrow\Lambda'$ such that
$\forall o\in\Omega$, $\exists h(o)\in\Omega'$ and $h(o)\models_{\Omega',\Lambda'} o$ and 
$\forall A\in\Lambda(o,o')$, $\exists k(A)\in \Lambda'(h(o),h(o'))$ and $k(A)\models_{\Omega',\Lambda'}A$.
\end{definition}
}{}

We can also define the closure of a network of ontologies by the network of ontologies which replaces each ontology by its $\omega$-closure and each
alignment by its $\alpha$-closure:
\begin{align*}
Cn(\langle\Omega,\Lambda\rangle  ) &= 
\langle  \{Cn^{\omega}_{\Omega, \Lambda } (o)\}_{o\in \Omega},
 \{Cn^{\alpha}_{\Omega, \Lambda}(o,o')\}_{o, o'\in\Omega} \rangle
\end{align*}
\onlyrr{
or alternatively:
\[ 
\langle  \{\{ \delta\mid \forall \langle m,\Delta\rangle\in\mathcal{M}(\langle\Omega, \Lambda\rangle), m_o\models\delta\}\}_{o\in\Omega},
\{\{ \mu\mid \forall \langle m,\Delta\rangle\in\mathcal{M}(\langle\Omega, \Lambda\rangle), m_o, m_{o'}\models_{\Delta}\mu\}\}_{o, o'\in\Omega}\rangle
 \]
}{}

We also use the notation $Cn^{\alpha}_{\langle\Omega,\Lambda\rangle}$ for $Cn^{\alpha}_{\Omega,\Lambda}$ and $Cn^{\omega}_{\langle\Omega,\Lambda\rangle}$ for $Cn^{\omega}_{\Omega,\Lambda}$.

\begin{example}[Full network closure]\label{ex:closure}
Here is the closure of the network of ontologies $\langle\Omega', \Lambda\rangle$ of Example~\ref{ex:inconsistency} (the first set is the syntactic form corresponding to the alignment or ontology, the second set is what is added by the local closure and the last set what is added by the $\omega$-closure or $\alpha$-closure):
\begin{align*}
Cn^{\omega}_{\Omega',\Lambda}(o_{1}) = Cn^{\omega}(o_{1}) =& \left\lbrace \begin{array}{r@{~}l} 
b_1 \sqsubseteq a_1, & c_1 \sqsubseteq a_1,\\ 
d_1 \sqsubseteq c_1, & e_1 \sqsubseteq c_1
\end{array} \right\rbrace
\bigcup
\left\lbrace \begin{array}{r@{~}l}
d_1 \sqsubseteq a_1, & e_1 \sqsubseteq a_1\\
\end{array} \right\rbrace\\
Cn^{\omega}_{\Omega',\Lambda}(o_{2}) = Cn^{\omega}(o_{2}) =& \left\lbrace \begin{array}{r@{~}l}
b_2 \sqsubseteq a_2, & c_2 \sqsubseteq a_2,\\
g_2 \sqsubseteq b_2, & f_2 \sqsubseteq b_2,\\
d_2 \sqsubseteq c_2, & e_2 \sqsubseteq c_2,
\end{array} \right\rbrace
\bigcup
\left\lbrace \begin{array}{r@{~}l}
d_2 \sqsubseteq a_2, & e_2 \sqsubseteq a_2,\\
f_2 \sqsubseteq a_2, & g_2 \sqsubseteq a_2 
\end{array} \right\rbrace\\
Cn^{\omega}_{\Omega',\Lambda}(o'_{3}) =& \left\lbrace \begin{array}{r@{~}l}
b_3 \sqsubseteq a_3,  & c_3 \sqsubseteq a_3,\\
g_3 \sqsubseteq b_3, & d_3 \sqsubseteq c_3,\\
e_3 \sqsubseteq c_3, & f_3 \sqsubseteq b_3,\\
b_3 \bot c_3 
\end{array} \right\rbrace
\bigcup
\left\lbrace \begin{array}{r@{~}l}
f_3 \sqsubseteq a_3, & g_3 \sqsubseteq a_3,\\
d_3 \sqsubseteq a_3, & e_3 \sqsubseteq a_3,\\
d_3 \bot b_3, & e_3 \bot b_3,\\
d_3 \bot f_3, & d_3 \bot g_3,\\
e_3 \bot f_3, & e_3 \bot g_3,\\
f_3 \bot c_3, & g_3 \bot c_3
\end{array} \right\rbrace
\bigcup
\{b_3\sqsupseteq e_3\}
\end{align*}
\noindent The network does not introduce new assertions in the two first ontologies, but the last one receives a new assertion.
Similarly, for alignments, their local closure does not provide new correspondences, but the $\alpha$-closure becomes larger.
These alignment closures are:
\begin{align*}
Cn^{\alpha}_{\Omega',\Lambda}(o_1,o_2)= & \left\lbrace \begin{array}{r}%
b_1 \leq d_2
\end{array} \right\rbrace
\bigcup 
\left\lbrace \begin{array}{r@{~~~}l}
b_1 \leq a_2, & b_1 \leq c_2 
\end{array} \right\rbrace\\
Cn^{\alpha}_{\Omega',\Lambda}(o_2,o'_3)= & \left\lbrace \begin{array}{r}
c_2 \leq b_3
\end{array} \right\rbrace
\bigcup 
\left\lbrace \begin{array}{r@{~~~}c@{~~~}l}
c_2 \leq a_3, & d_2 \leq a_3, & e_2 \leq a_3 \\
d_2 \leq b_3, & e_2 \leq b_3, \\
c_2 \bot c_3,  & c_2 \bot d_3,  & c_2 \bot e_3 \\
d_2 \bot c_3,  & d_2 \bot d_3,  & d_2 \bot e_3 \\
e_2 \bot c_3,  & e_2 \bot d_3,  & e_2 \bot e_3 \\
d_2 \geq e_3, & c_2 \geq e_3, & a_2 \geq e_3
\end{array} \right\rbrace\\
Cn^{\alpha}_{\Omega',\Lambda}(o_1,o'_3)= & \left\lbrace \begin{array}{r}
e_1 \geq f_3,\\
b_1 \geq e_3
\end{array} \right\rbrace
\bigcup
\left\lbrace \begin{array}{r@{~~~}c@{~~~}l}
c_1 \geq f_3, & a_1 \geq f_3, & a_1 \geq e_3,\\
b_1 \bot c_3, & b_1 \bot d_3, & b_1 \bot e_3,\\
b_1 \leq b_3, & b_1 \leq a_3
\end{array} \right\rbrace
\end{align*}
Such a representation is highly redundant as closures usually are.
\end{example}

The closure of a network of ontologies may introduce non empty alignments between ontologies which were not previously connected or empty. 
This is possible because constraints do not come locally from the alignment but from the whole network of ontologies.
Such a formalism contributes to the definition of the meaning of alignments: 
it describes what are the consequences of ontologies with alignments, i.e., what can be deduced by an agent.

\onlyrr{
\begin{property}\label{prop:closure}
$Cn$ is a closure operation on normalised networks of ontologies\footnote{A closure operation $Cn$ in a set $S$ satisfies three properties: 
$\forall X, Y\in S:$ $X\subseteq Cn(X)$, 
$Cn(X)=Cn(Cn(X))$, and
$X\subseteq Y\Rightarrow Cn(X)\subseteq Cn(Y)$.}.
\end{property}

This property can rely on a lemma mirroring Property~\ref{prop:monomod}.
\begin{lemma}[Consequence isotony]\label{prop:monocons}
Let $\langle\Omega, \Lambda\rangle$ and $\langle\Omega', \Lambda'\rangle$ be two networks of ontologies,
$$\langle\Omega,\Lambda\rangle \sqsubseteq \langle\Omega',\Lambda'\rangle \Rightarrow 
Cn(\langle\Omega,\Lambda\rangle) \sqsubseteq Cn(\langle\Omega',\Lambda'\rangle)$$
\end{lemma}

\begin{proof}[Proof of Lemma~\ref{prop:monocons}]
From Property~\ref{prop:monomod}, we know that there exists $h:\Omega\rightarrow\Omega'$ and that 
$\forall \langle m',\Delta'\rangle\in \mathcal{M}(\langle\Omega',\Lambda'\rangle)$,
$\exists  \langle m,\Delta\rangle\in \mathcal{M}(\langle\Omega,\Lambda\rangle)$;
$\forall o\in\Omega, m_o=m'_{h(o)}$ and $\Delta\ll\Delta'$.
Hence, $\exists h:\Omega\rightarrow\Omega'$ such that
$\forall o\in\Omega$, each model of $h(o)$ is a model of $o$, so 
$\{m_o\mid \langle m,\Delta\rangle\in\mathcal{M}(\langle\Omega, \Lambda\rangle)\}\supseteq \{m'_{h(o)}\mid \langle m',\Delta'\rangle\in\mathcal{M}(\langle\Omega', \Lambda'\rangle)\}$.
This entails $Cn^{\omega}_{\Omega, \Lambda}(o)\subseteq Cn^{\omega}_{\Omega', \Lambda'}(h(o))$.
This also means that $\forall o, o'\in\Omega$, each pair of models $\langle m'_{h(o)}, m'_{h(o')}\rangle$ of $\langle h(o), h(o')\rangle$ 
is also a pair of models of $\langle o, o'\rangle$.
Considering $\mu\in Cn^{\alpha}_{\Omega, \Lambda}(o,o')$, either $\forall \langle m',\Delta'\rangle\in\mathcal{M}(\langle\Omega', \Lambda'\rangle)$, $m'_{h(o)}, m'_{h(o')}\models_{\Delta'}\mu$ and then $\mu\in Cn^{\alpha}_{\Omega', \Lambda'}(h(o),h(o'))$,
or $\exists \langle m',\Delta'\rangle\in\mathcal{M}(\langle\Omega', \Lambda'\rangle)$; $m'_{h(o)}, m'_{h(o')}\not\models_{\Delta'}\mu$.
In this latter case, $\exists \langle m,\Delta'\rangle\in\mathcal{M}(\langle\Omega, \Lambda\rangle)$, such that $m_o=m'_{h(o)}$.
This is a model of $\langle\Omega, \Lambda\rangle$ because this network put less constraints on $\Delta$ than $\langle\Omega', \Lambda'\rangle$.
But this contradicts the hypothesis that $\mu\in Cn^{\alpha}_{\Omega, \Lambda}(o,o')$.
Hence, $Cn^{\alpha}_{\Omega, \Lambda}(o,o')\subseteq Cn^{\alpha}_{\Omega', \Lambda'}(h(o),h(o'))$, so 
$Cn(\langle\Omega,\Lambda\rangle) \sqsubseteq Cn(\langle\Omega',\Lambda'\rangle)$.
\end{proof}

\begin{proof}[Proof of Property~\ref{prop:closure}]
$\forall \langle\Omega,\Lambda\rangle$ and $\langle\Omega',\Lambda'\rangle$ normalised networks of ontologies: 
\begin{itemize*}

\item $\langle\Omega,\Lambda\rangle\sqsubseteq Cn(\langle\Omega,\Lambda\rangle)$, because 
$\forall \langle h,k\rangle$ such that $\forall o,o'\in\Omega$,
$h(o)=Cn^{\omega}_{\Omega,\Lambda}(o)$ and $k(\lambda(o,o'))=Cn^{\alpha}_{\Omega,\Lambda}(o,o')$,
\begin{inparaenum}[($i$)]
\item $o\subseteq Cn^{\omega}_{\Omega,\Lambda}(o)$, because $\forall \delta\in o, \forall \langle m,\Delta\rangle\in \mathcal{M}(\langle\Omega,\Lambda\rangle), m_o\models o$, hence $m_o\models\delta$, so $\delta\in Cn^{\omega}_{\Omega,\Lambda}(o)$; and
\item $\lambda(o,o')\subseteq Cn^{\alpha}_{\Omega,\Lambda}(o,o')$, because 
$\forall\mu\in \lambda(o,o')$, $\forall \langle m,\Delta\rangle\in \mathcal{M}(\langle\Omega,\Lambda\rangle)$, $m_o, m_{o'}\models_{\Delta}\mu$, hence $\mu\in Cn^{\alpha}_{\Omega,\Lambda}(\lambda(o,o'))$.
\end{inparaenum}

\item 
$\forall \langle m,\Delta\rangle\in \mathcal{M}(\langle\Omega, \Lambda\rangle)$, 
\begin{inparaenum}[($i$)]
\item $\forall o\in\Omega, m_o\in\mathcal{M}(o)$ and $\forall o, o'\in\Omega, m_o, m_{o'}\models_{\Delta} \lambda(o, o')$, 
\item $\forall o\in\Omega$, $m_o\in\mathcal{M}(Cn^{\omega}_{\Omega,\Lambda}(o))$, and 
\item $\forall o, o'\in\Omega$, $m_o, m_{o'}\models_{\Delta} Cn^{\alpha}_{\Omega,\Lambda}(o, o')$ 
(the two latter assertions because the closure contains elements true in all models).
\end{inparaenum}
In consequence, $\langle m,\Delta\rangle\in \mathcal{M}(Cn(\langle\Omega,\Lambda\rangle))$, which means that $\mathcal{M}(\langle\Omega, \Lambda\rangle)\subseteq\mathcal{M}(Cn(\langle\Omega,\Lambda\rangle))$, 
and thus
$Cn(\langle\Omega,\Lambda\rangle)\sqsupseteq Cn(Cn(\langle\Omega,\Lambda\rangle))$ (less models means more consequences).
By the first clause, $Cn(\langle\Omega,\Lambda\rangle)\sqsubseteq Cn(Cn(\langle\Omega,\Lambda\rangle))$,
so, $Cn(\langle\Omega,\Lambda\rangle)\equiv Cn(Cn(\langle\Omega,\Lambda\rangle))$.

\item Consequence isotony is proved by Lemma~\ref{prop:monocons}

\end{itemize*}
\end{proof}

\begin{example}
Figure~\ref{fig:extreme} displays an extreme example of a network.
This network is inconsistent in the interpretation of Example~\ref{ex:relations}, though none of its ontologies nor alignments is inconsistent.
The inconsistency manifests itself by starting with the network without one of the correspondences and revising it by this correspondence. 
It can be only solved by suppressing one of the \emph{other} correspondences.

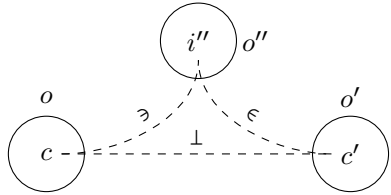
\begin{figure}[!h]
\begin{center}
\begin{tikzpicture}[thin] 

\draw (0,0) node (c) {$c$};
\draw (4,0) node (cp) {$c'$};
\draw (2,1.5) node (ipp) {$i''$};
\draw[dashed] (c.east) -- node[above] {$\bot$} (cp.west);
\draw[dashed] (cp) .. controls +(-1,0)  and +(0,-1) ..  node[above,sloped] {$\in$} (ipp);
\draw[dashed] (c) .. controls +(1,0)  and +(0,-1) ..  node[above,sloped] {$\ni$} (ipp);

\draw (0,0) ellipse (.5cm and .5cm);
\draw (0,.75) node {$o$};
\draw (4,0) ellipse (.5cm and .5cm);
\draw (4,.75) node {$o'$};
\draw (2,1.5) ellipse (.5cm and .5cm);
\draw (2.75,1.5) node {$o''$};

\end{tikzpicture}
\end{center}
\caption{Globally inconsistent alignment pattern.}\label{fig:extreme}
\end{figure}
\end{example}

}{

\section{Semantic constraints}\label{sec:constraints}

There has been many semantics proposed for alignments and networks of ontologies.
For the sake of simplicity, we will only rely on the abstract framework given in Section~\ref{sec:abssem} as long as the semantics satisfies the three following properties:
\begin{description}
\item[There exists a sound normalisation,] i.e., each network of ontology can be represented as a normalised network having exactly the same set of models;
\begin{enumerate}
\item $\forall \langle \Omega, \Lambda\rangle$, $\exists \langle \Omega, \Lambda'\rangle$ such that $\forall o, o'\in\Omega$, $|\Lambda'(o,o')|=1$ and $\mathcal{M}(\langle \Omega, \Lambda\rangle)=\mathcal{M}(\langle \Omega, \Lambda'\rangle)$.
\end{enumerate}
\item[Downward consistency preservation] on normalised networks, so that if a network is subsumed by a consistent network, then it is consistent: 
\begin{enumerate}[resume]
\item If $\langle \Omega, \Lambda\rangle \sqsubseteq \langle \Omega', \Lambda'\rangle$ and $\langle \Omega', \Lambda'\rangle$ is consistent, then $\langle \Omega, \Lambda\rangle$ is consistent
\end{enumerate}
\item[$Cn$ is a closure operation] on normalised networks:
\begin{enumerate}[resume]
\item $\langle \Omega, \Lambda\rangle \sqsubseteq Cn(\langle \Omega, \Lambda\rangle)$
\item $Cn(\langle \Omega, \Lambda\rangle) \sqsubseteq Cn(Cn(\langle \Omega, \Lambda\rangle))$
\item $\langle \Omega, \Lambda\rangle \sqsubseteq \langle \Omega', \Lambda'\rangle \Rightarrow Cn(\langle \Omega, \Lambda\rangle) \sqsubseteq Cn(\langle \Omega', \Lambda'\rangle)$
\end{enumerate}
\end{description}
These are rather natural properties.
The remainder of the paper relies only on these properties.
In \cite{euzenat2014d}, we show that these constraints are satisfied by those semantics defined as in Section~\ref{sec:abssem}.

}

\newpage
\section{The categories of networks of ontologies}\label{sec:categ}

We introduce categories of networks of ontologies.
This allows to consider networks of ontologies from a more abstract perspective and to take advantage of general operations, such as pullbacks.

\subsection{Syntactic category of networks of ontologies}\label{sec:syncateg}

Networks of ontologies, together with the morphisms previously defined form a category.


\begin{definition}[Category $\mathcal{NOO}$]
Let $\mathcal{NOO}$ be the structure made of:
\begin{itemize*}
\item objects are networks of ontologies as in Definition~\ref{def:noo};
\item morphisms are morphisms between networks of ontologies as in Definition~\ref{def:synmorph};
\end{itemize*}
such that
\begin{itemize*}
\item the composition of morphisms is simply function composition: $\langle h, k\rangle\circ\langle h', k'\rangle=\langle h\circ h', k\circ k'\rangle$;
\item the identity morphism $1_{\langle \Omega, \Lambda\rangle}=\langle 1_{\Omega}, 1_{\Lambda}\rangle$ is defined as the morphism associating an ontology to itself and an alignment to itself.
\end{itemize*}
\end{definition}


\begin{property}
$\mathcal{NOO}$ is a category
\end{property}
\begin{proof}
The proof directly follows from the definition:
Given a morphism $\langle h, k\rangle$ between networks $\langle\Omega, \Lambda\rangle$ and $\langle\Omega', \Lambda'\rangle$ and $\langle h', k'\rangle$ between networks $\langle\Omega', \Lambda'\rangle$ and $\langle\Omega'', \Lambda''\rangle$, the result of their composition $\langle h'\circ h, k'\circ k\rangle$ is a morphism between  $\langle\Omega, \Lambda\rangle$ and $\langle\Omega'', \Lambda''\rangle$.
Indeed, If $\forall o\in\Omega$, $o\subseteq h(o)$, and $\forall o'\in\Omega'$, $o'\subseteq h'(o')$, then $o\subseteq h(h'(o))=h'\circ h(o)$.
If $\forall A\in\Lambda(o,p)$, $k(A)\in\Lambda'(h(o),h(p))$ and $\forall A'\in\Lambda'(o',p')$, $k'(A')\in\Lambda''(h'(o'),h(p'))$
then $\forall A\in\Lambda(o,p)$, $k'\circ k(A)\in\Lambda'(h'\circ h(o),h'\circ h(p))$.
If $\forall A\in\Lambda(o,p)$, $A\subseteq k(A)$ and $\forall A'\subseteq\Lambda'(o',p')$, $A'\in k'(A')$
then $\forall A\in\Lambda(o,p)$, $A\subseteq k(A)\subseteq k'(k(A))=k'\circ k(A)$.

Composition is associative, i.e., if $f: A \rightarrow B$, $g: B \rightarrow C$, $h: C \rightarrow D$, then $h \circ ( g \circ f ) = (h \circ g) \circ f$, simply because function composition is associative.

The identity morphism is indeed a morphism, because $\forall o\in\Omega$, $o\subseteq o=1_{\Omega}(o)$ and $\forall A\in\Lambda(o,p)$, $1_{\Lambda}(A)=A\in \Lambda(o,p)=\Lambda(1_{\Omega}(o),1_{\Omega}(p))$. Moreover, $A\subseteq A=1_{\Lambda}(A)$.

Finally, $\forall \langle h, k\rangle$ between networks $\langle\Omega, \Lambda\rangle$ and $\langle\Omega', \Lambda'\rangle$,
$1_{\langle\Omega, \Lambda\rangle}\circ \langle h, k\rangle = \langle 1_{\Omega}\circ h, 1_{\Lambda}\circ k\rangle = \langle h, k\rangle = \langle h\circ 1_{\Omega'}, k\circ 1_{\Lambda'}\rangle = \langle h, k\rangle\circ 1_{\langle\Omega', \Lambda'\rangle}$.
\end{proof}

A conjunction (meet) operator between networks of ontologies can be introduced generally in a standard way from subsumption, as the greatest common subsumee, but it would not be necessarily unique. 
However, it is possible to introduce a fibred meet, denoted $\overline{\sqcap}$, which define the conjunction of a set of networks along a set of isomorphisms to a common generating network.

\begin{definition}[Fibred meet of networks of ontologies]
Given a network of ontologies $\langle\Omega, \Lambda\rangle$
and a finite family of networks of ontologies, $\{\langle\Omega_j, \Lambda_j\rangle\}_{j\in J}$, 
such that $\exists \langle h_j, k_j\rangle_{j\in J}$, pairs of 
isomorphisms:
$h_j:\Omega_j\longrightarrow\Omega$ and $k_j:\Lambda_j\longrightarrow\Lambda$ 
with $\forall A\in\Lambda(o,o')$, $k_j(A)\in \Lambda(h_j(o),h_j(o'))$, the \emph{fibred meet} of $\{\langle\Omega_j, \Lambda_j\rangle\}_{j\in J}$ with respect to $\langle h_j, k_j\rangle_{j\in J}$ is
$$\overline{\bigsqcap_{j\in J}}\langle \Omega_j,\Lambda_j\rangle = \langle \{ \bigcap_{j\in J} h_j^{-1}(o)\}_{o\in\Omega}, \{\bigcap_{j\in J} k_j^{-1}(A)\}_{A\in\Lambda} \rangle$$
\end{definition}

We may denote $\overline{\bigsqcap}_{j\in J}\langle \Omega_j,\Lambda_j\rangle$ as $\langle \overline{\Omega}, \overline{\Lambda}\rangle$.

The generating morphisms must be isomorphisms because, if they were not injective, the $\pi_j$ cannot be defined (there would be two candidate targets), and if they were not surjective, they would not necessarily cover the whole source network of alignment (see Figure~\ref{fig:iso}).

The degenerated case of a fibred meet is when all ontologies are empty (and hence alignments are empty as well).

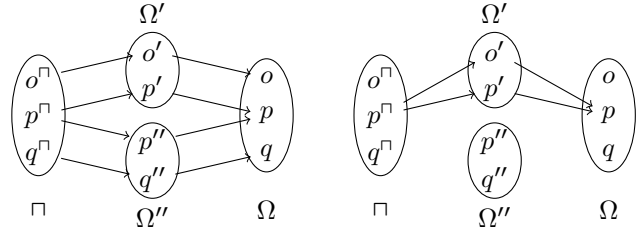
\begin{figure}[!h]
\begin{center}
\begin{tikzpicture}[thin,scale=.5] 

\draw (0,3) node (a1) {$o^{\sqcap}$};
\draw (0,2) node (b1) {$p^{\sqcap}$};
\draw (0,1) node (c1) {$q^{\sqcap}$};
\draw (0,2) ellipse (.7cm and 1.6cm);
\draw (0,-.5) node () {$\sqcap$};

\draw (3,4.7) node () {$\Omega'$};
\draw (3,3.7) node (a2) {$o'$};
\draw (3,2.7) node (b2) {$p'$};
\draw (3,3.2) ellipse (.7cm and 1cm);

\draw (3,1.3) node (b3) {$p''$};
\draw (3,0.3) node (c3) {$q''$};
\draw (3,.8) ellipse (.7cm and 1cm);
\draw (3,-.7) node () {$\Omega''$};

\draw (6,3) node (a4) {$o$};
\draw (6,2) node (b4) {$p$};
\draw (6,1) node (c4) {$q$};
\draw (6,2) ellipse (.7cm and 1.5cm);
\draw (6,-.5) node () {$\Omega$};

\draw[->] (a1) -- (a2);
\draw[->] (b1) -- (b2);
\draw[->] (b1) -- (b3);
\draw[->] (c1) -- (c3);

\draw[->] (a2) -- (a4);
\draw[->] (b2) -- (b4);
\draw[->] (b3) -- (b4);
\draw[->] (c3) -- (c4);

\begin{scope}[xshift=9cm]

\draw (0,3) node (a1) {$o^{\sqcap}$};
\draw (0,2) node (b1) {$p^{\sqcap}$};
\draw (0,1) node (c1) {$q^{\sqcap}$};
\draw (0,2) ellipse (.7cm and 1.6cm);
\draw (0,-.5) node () {$\sqcap$};

\draw (3,4.7) node () {$\Omega'$};
\draw (3,3.7) node (a2) {$o'$};
\draw (3,2.7) node (b2) {$p'$};
\draw (3,3.2) ellipse (.7cm and 1cm);

\draw (3,1.3) node (b3) {$p''$};
\draw (3,0.3) node (c3) {$q''$};
\draw (3,.8) ellipse (.7cm and 1cm);
\draw (3,-.7) node () {$\Omega''$};

\draw (6,3) node (a4) {$o$};
\draw (6,2) node (b4) {$p$};
\draw (6,1) node (c4) {$q$};
\draw (6,2) ellipse (.7cm and 1.5cm);
\draw (6,-.5) node () {$\Omega$};

\draw[->] (b1) -- (a2);
\draw[->] (b1) -- (b2);

\draw[->] (a2) -- (b4);
\draw[->] (b2) -- (b4);

\end{scope}

\end{tikzpicture}
\end{center}
\caption{Two cases of non isomorphic generators (ellipsis represents sets of ontologies, letters represent ontologies). 
On the left, the generating morphisms are not surjective, hence functions generated from $\sqcap$ to $\Omega'$ and $\Omega''$ are not morphisms (they do not cover $\sqcap$). 
On the right, the generating morphism is not injective, so it is not possible to create a morphism from $\sqcap$ to $\Omega'$.}\label{fig:iso}
\end{figure}

The fibred meet is a pullback in the $\mathcal{NOO}$ category.

\begin{property}
The fibred meet is a pullback for the $\mathcal{NOO}$ category.
\end{property}
\begin{proof}
$\pi_{j'}=\langle h'_{j'}, k'_{j'}\rangle$ such that 
$h'_{j'}(\bigcap_{j\in J} h_j^{-1}(o))=h_j^{-1}(o)$, and
$k'_{j'}(\bigcap_{j\in J} k_j^{-1}(A))=k_j^{-1}(A)$.
This is a morphism since (1) $\bigcap_{j\in J} h_j^{-1}(o)\subseteq h_{j'}^{-1}(o)$,
(2) $\bigcap_{j\in J} k_j^{-1}(A)\subseteq k_{j'}^{-1}(A)$, and
(3) if $\bigcap_{j\in J} k_j^{-1}(A)\in\overline{\Lambda}(\bigcap_{j\in J} h_j^{-1}(o), \bigcap_{j\in J} h_j^{-1}(o'))$, 
 then $k'_{j'}(\bigcap_{j\in J} k_j^{-1}(A))\in\Lambda_{j'}( h_{j'}^{-1}(o), h_{j'}^{-1}(o'))=\Lambda_{j'}( h'_{j'}(\bigcap_{j\in J} h_j^{-1}(o))$, $h'_{j'}(\bigcap_{j\in J} h_j^{-1}(o'))$.

Moreover, the diagram commutes, i.e., $\forall j', j''\in J$, $\langle h_{j'}, k_{j'}\rangle\circ\pi_{j'} = \langle h_{j''}, k_{j''}\rangle\circ\pi_{j''}$, because
$h'_{j'}\circ h_{j'}(\bigcap_{j\in J} h_j^{-1}(o)) = o = h'_{j''}\circ h_{j''}(\bigcap_{j\in J} h_j^{-1}(o))$ and
$k'_{j'}\circ k_{j'}(\bigcap_{j\in J} k_j^{-1}(A)) = A = k'_{j''}\circ k_{j''}(\bigcap_{j\in J} k_j^{-1}(A))$
because $\langle h_j, k_j\rangle$ are pairs of isomorphism.

Finally, $\forall \theta_{j'}(x)=h_{j'}^{-1}(o)$, then $u$ can be defined such that $u(x)=\bigcap_{j\in J} h^{-1}_j(o)$.
Similarly, $\forall \theta_{j'}(Y)=k_{j'}^{-1}(A)$, then $u(Y)$ can be defined as $u(Y)=\bigcap_{j\in J} h^{-1}_j(A)$
then $\theta_{j'}(Y)=u\circ\pi_{j'}(Y)$, so the fibred meet is universal.
\end{proof}

\begin{center} 
\begin{tikzpicture}

\draw (-1,3) node (im) {$\langle\Omega', \Lambda'\rangle$};
\draw (2,3) node (prod) {$\overline{\bigsqcap}_{j\in J}\langle \Omega_j,\Lambda_j\rangle$};
\draw (2,1) node (im1) {$\langle\Omega_j, \Lambda_j\rangle$};
\draw (5,1) node (im0) {$\langle\Omega, \Lambda\rangle$};

\draw[->] (im) -- node[left] {$\theta_j$} (im1); 
\draw[->] (prod) -- node[right] {$\pi_j$} (im1); 
\draw[->] (im1) -- node[above] {$\langle h_j, k_j\rangle$} (im0);
\draw[->,dashed] (im) -- node[above]{$u$} (prod);

%
%
\end{tikzpicture}
\end{center} 

Such a fibred meet is well-defined and unique up to isomorphism for any set of subnetworks of a particular generator network.
It does not exists for any diagrams with morphisms to an object, the morphisms have to be isomorphic as presented in Figure~\ref{fig:iso}.

It satisfies the premises of Property~\ref{prop:meetincl}.

\begin{property}\label{prop:meetincl}
Let $\langle\Omega, \Lambda\rangle$ be a network of ontologies
and $\{\langle\Omega_j, \Lambda_j\rangle\}_{j\in J}$ a finite family of networks of ontologies
whose fibred meet with respect to $\langle\Omega, \Lambda\rangle$ is defined through morphisms $\langle h_j, k_j\rangle$.
%
$$\forall j\in J,\quad \overline{\bigsqcap_{j'\in J}}\langle \Omega_{j'},\Lambda_{j'}\rangle ~\sqsubseteq~ \langle \Omega_j,\Lambda_j\rangle ~\sqsubseteq~ \langle \Omega,\Lambda\rangle$$
\end{property}
\onlyrr{
\begin{proof}
$\langle \Omega_j,\Lambda_j\rangle \sqsubseteq \langle \Omega,\Lambda\rangle$ is straightforward because the generating pairs of morphism $\langle h_j, k_j\rangle$ are already such that this is true.

\noindent $\overline{\bigsqcap}_{j'\in J}\langle \Omega_{j'},\Lambda_{j'}\rangle \sqsubseteq \langle \Omega_j,\Lambda_j\rangle$ because
it is possible to consider the pair of morphisms $\langle h'_j, k'_j\rangle$ such that
$h'_j(\bigcap_{j'\in J} h_{j'}^{-1}(o)) = h^{-1}_j(o)$ (and obviously, $\bigcap_{j'\in J} h_{j'}^{-1}(o)\subseteq h^{-1}_j(o)$).
Similarly, 
$k'_j(\bigcap_{j'\in J} k_{j'}^{-1}(A)) = k_{j}^{-1}(A)$ and again, $\bigcap_{j'\in J} k_{j'}^{-1}(A) \subseteq k_{j}^{-1}(A)$.
These are graph morphisms because all $k_{j'}^{-1}(A)$ are alignments between the precursors of the ontologies of $A$.
\end{proof}
}{}

The 
fibred meet generated by consistent subnetworks is always consistent since it is subsumed by consistent networks (downward consistency preservation).

Moreover, if the ontologies in the networks are closed, then the fibred meet is closed as well (because it intersects closed networks).

\subsection{Category of weighted networks of ontologies}\label{sec:wcateg}

This category can be generalised for taking weights into account \cite{atencia2011a}.
Following the presentation of \cite{euzenat2013c},
the relationship between two entities can be assigned a degree of confidence, or weight, taken from a confidence structure.

\begin{definition}[Confidence structure]\index{confidence!degree|emph}
\index{confidence!structure|emph}\index{degree!confidence -|emph}\index{$\Xi$ (confidence structure)|emph}
A confidence structure is an ordered set of degrees $\langle \Xi, \leq\rangle$ for which there exists
the greatest element $\top$ and the smallest element $\bot$.
\end{definition}

The usage of confidence degrees is such that the higher the degree with regard to $\leq$, the more likely
the relation holds.
It is convenient to interpret the greatest element as the Boolean true and the least element as
the Boolean false. 
The function $\kappa_A(\mu): Q_L(o)\times Q'_{L'}(o')\times\Theta\rightarrow \Xi$\index{$\kappa$ (confidence function)|emph} provides the confidence for a correspondence $\mu$ in an alignment $A$.

Networks of ontologies may be defined on weighted alignments, by simply integrating $\kappa_A$ in $A$.
Morphisms can be defined by taking into account the confidence (or weight) associated to correspondences.

\begin{definition}[Weight-aware morphism between networks of ontologies]\label{def:weightmorph}
Given two networks of ontologies, $\langle\Omega, \Lambda\rangle$ and $\langle \Omega',\Lambda'\rangle$,
a \emph{weight-aware morphism} between $\langle\Omega, \Lambda\rangle$ and $\langle \Omega',\Lambda'\rangle$, is
$\langle h, k\rangle$, a pair of morphisms: $h:\Omega\longrightarrow\Omega'$ and $k:\Lambda\longrightarrow\Lambda'$ such that
$\forall o\in\Omega$, $\exists h(o)\in\Omega'$ and $o\subseteq h(o)$ and 
$\forall A\in\Lambda(o,o'), \forall \mu\in A$, $\exists k(A)\in \Lambda'(h(o),h(o'))$ and $\mu'\in k(A)$ such that $\mu=\mu'$ and $\kappa_A(\mu)\leq \kappa_{k(A)}(\mu')$.
\end{definition}

They can be used for defining a new category:

\begin{definition}[Category $\mathcal{NOO_W}$]
Let $\mathcal{NOO_W}$ be the structure made of:
\begin{itemize*}
\item objects are networks of ontologies as in Definition~\ref{def:noo} with weighted alignments;
\item morphisms are weight-aware morphisms between networks of ontologies as in Definition~\ref{def:weightmorph};
\end{itemize*}
\end{definition}

Obviously this is a category again, we omit the proof.

There is a natural family of functors $F_{w}$, with $w\in\Xi$ a threshold value, which are simply quotientation functions.

\begin{definition}[$F_w$ function]
The function $F_w: \mathcal{NOO_W}\rightarrow \mathcal{NOO}$ with $w$ within the codomain of $\kappa$ defined as:
\begin{align*}
F_{w}(\langle\Omega, \Lambda\rangle) &=\langle\Omega, \{ F_{w}(A)\mid A\in \Lambda\}\rangle\\
F_{w}(\langle h, k\rangle) &=\langle h, F_{w}(k)\rangle\\
F_{w}(A) &= \{ \mu\in A\mid \kappa(\mu)\geq w\}\\
F_{w}(k)(\mu) &= k(\mu)
\end{align*}
\end{definition}

\begin{property}
$F_w$ is a functor
\end{property}
\begin{proof}
Clearly, $\forall k: \Lambda\rightarrow\Lambda'$,  $F_{w}(k): F_{w}(\Lambda)\rightarrow F_{w}(\Lambda')$,
$F_{w}(k)(F_{w}(A)) = \{ \mu\in k(A)\mid \kappa(\mu)\geq w\} = F_{w}(k(A))$.
Moreover, if $k(A)\in \Lambda'$, then $F_{w}(k(A))\in F_{w}(\Lambda)$.
Finally, $F_{w}(k)$ is a (weight-blind) morphism because, if $\mu\in F_{w}(A)$, then $\kappa(\mu)\leq \kappa(k(\mu))$,
so, $k(\mu)\in F_{w}(k)(A)$.
\end{proof}

These functors are both applying thresholds to the content of alignments and ignoring the weights, it is also possible, with similar arguments to define functors which perform both functions independently (they do not commute: weight-ignorance can only be applied at the end of the chain).


\subsection{Further refinements}

Further refinements may be introduced such as considering that relations are not independent, e.g., that $\langle e, =, e'\rangle$ entails $\langle e, \sqsubseteq, e'\rangle$ and thus that morphisms should be defined by taking this into account.
This is slightly more difficult because several correspondences may be equated to a single correspondence. This approach may be refined by using algebras of relations \cite{euzenat2008e}.

Finally, considering semantic morphisms would complete the picture.

\section{Related work}\label{sec:rwork}

Obviously the work presented in (\S\ref{sec:abssem}) is related to this one focussing on one specific semantics for alignments and eventually for ontologies.

In addition, the work around DOL \cite{omg2014a} is a wider effort building from the ground (logic) up to ontologies and alignments through describing the semantics of ontology and alignment languages.
Instead, we start from these ontologies and alignments and abstract from their underlying semantics to determine properties on top of which general purpose operations can be built.
Hopefully, these two efforts will meet.

\section{Conclusion}

This report brought two modest contributions:
\begin{itemize*}
\item a semantics for networks of ontologies which is parameterized by actual semantics and properties independent from the chosen semantics;
\item an abstraction of networks of ontologies into categories on which pullbacks can be defined.
\end{itemize*}
These contributions provide general properties that can be used for defining concrete operations.
We used them for defining revision of networks of ontologies.

The category of networks have been defined purely syntactically with a disconnection to the semantics.
Doing it semantically would be useful as well.

\section*{Acknowledgements}

I am thankful to Antoine Zimmermann, Manuel Atencia and Armen Inants for fruitful discussions.


\bibliographystyle{plain}
\bibliography{aij-revision}

\newpage
\tableofcontents

\end{document}